\documentclass[preprint]{imsart}

\RequirePackage[OT1]{fontenc}
\RequirePackage{amsthm,amsmath,amssymb}
\RequirePackage[round]{natbib}
\RequirePackage[colorlinks,citecolor=blue,urlcolor=blue]{hyperref}
\RequirePackage{fullpage}
\RequirePackage{todonotes}
\RequirePackage{bm}
\RequirePackage{booktabs}
\RequirePackage{comment}
\RequirePackage{subcaption}
\RequirePackage{algorithm}
\RequirePackage[noend]{algpseudocode}

\usepackage{nicefrac}         
\usepackage{microtype}      
\usepackage{forest}
\forestset{sn edges/.style={for tree={edge={->}}}} 
\usepackage{tikz}
\usetikzlibrary{positioning,angles,quotes}
\usepackage[cmintegrals]{newtxmath}
\usepackage[cal=euler]{mathalfa}
\usepackage{enumitem}
\usepackage{caption,subcaption}
\usepackage{bm}
\usepackage{wrapfig}
\usepackage{macros}
\usepackage{todonotes}
\usepackage{mathtools}

\usepackage[none]{hyphenat}

\newcommand{\E}{\mathbb{E}}
\newtheorem{prop}{Proposition}
\newtheorem{remark}{Remark}

\usepackage{xr}

\DeclarePairedDelimiter\abs{\lvert}{\rvert}
\DeclarePairedDelimiter\norm{\lVert}{\rVert}%
\makeatletter
\let\oldabs\abs
\def\abs{\@ifstar{\oldabs}{\oldabs*}}
\let\oldnorm\norm
\def\norm{\@ifstar{\oldnorm}{\oldnorm*}}
\makeatother

\usepackage{macros}

\arxiv{arXiv:1906.10228 }

\begin{document}
\sloppy

\begin{frontmatter}
\title{A Theoretical Connection Between \\ Statistical Physics and Reinforcement Learning}
\runtitle{A Theoretical Connection Between Statistical Physics and Reinforcement Learning}

\begin{aug}

\author{\fnms{Jad} \snm{Rahme}
 } 
\and %
\author{\fnms{Ryan P.} \snm{Adams}
} %
%

\address{Princeton University}

\footnotesize{\texttt{\{jrahme,rpa\}@princeton.edu}}

\runauthor{ J. Rahme  and R.P. Adams}

\end{aug}

\begin{abstract}
Sequential decision making in the presence of uncertainty and stochastic dynamics gives rise to distributions over state/action trajectories in reinforcement learning (RL) and optimal control problems.
This observation has led to a variety of connections between RL and inference in probabilistic graphical models (PGMs).
Here we explore a different dimension to this relationship, examining reinforcement learning using the tools and abstractions of statistical physics.
The central object in the statistical physics abstraction is the idea of a partition function~$\mcZ$, and here we construct a partition function from the ensemble of possible trajectories that an agent might take in a Markov decision process.
Although value functions and~$Q$-functions can be derived from this partition function and interpreted via average energies, the~$\mcZ$-function provides an object with its own Bellman equation that can form the basis of alternative dynamic programming approaches.
Moreover, when the MDP dynamics are deterministic, the Bellman equation for~$\mcZ$ is linear, allowing direct solutions that are unavailable for the nonlinear equations associated with traditional value functions. 
The policies learned via these~$\mcZ$-based Bellman updates are tightly linked to Boltzmann-like policy parameterizations. In addition to sampling actions proportionally to the exponential of the expected cumulative reward as Boltzmann policies would, these policies take {\it entropy} into account favoring states from which many outcomes are possible.  
\end{abstract}

\end{frontmatter}

\maketitle

\section{Introduction}

One of the central challenges in the pursuit of machine intelligence is robust sequential decision making.
In a stochastic and uncertain environment, an agent must capture information about the distribution over ways they may act and move through the state space.
Indeed, the algorithmic process of planning and learning itself can lead to a well-defined distribution over state/action trajectories.
This observation has led to a variety of connections between reinforcement learning (RL) and inference in probabilistic graphical models (PGMs) \citep{Levine18}.
In some ways this connection is unsurprising: belief propagation (and its relatives such as the sum-product algorithm) is understood to be an example of dynamic programming \citep{koller2009probabilistic} and dynamic programming was developed to solve control problems \citep{bellman1966dynamic, bertsekas1995dynamic}.
Nevertheless, the exploration of the connection between control and inference has yielded fruitful insights into sequential decision making algorithms \citep{kalman1960new,attias2003planning,
ziebart2010modeling, kappen2011optimal,Levine18}.

In this work, we present another point of view on reinforcement learning as a distribution over trajectories, one in which we draw upon useful abstractions from statistical physics.
This view is in some ways a natural continuation of the agenda of connecting control to inference, as many insights in probabilistic graphical models have deep connections to, e.g., spin glass systems \citep{hopfield1982neural,yedidia2001generalized,zdeborova2016statistical}. 
More generally, physics has often been a source of inspiration for ideas in machine learning \citep{mackay2003information, mezard2009information}.
Boltzmann machines \citep{ackley85}, Hamiltonian Monte Carlo \citep{duane1987hybrid, neal2011mcmc,betancourt2017conceptual} and, more recently, tensor networks \citep{stoudenmire2016supervised} are a few examples.
In addition to direct inspiration, physics provides a compelling framework to reason about certain problems.
The terms \emph{momentum}, \emph{energy}, \emph{entropy}, and \emph{phase transition} are ubiquitous in machine learning.
However, abstractions from physics have generally not been so far helpful for understanding reinforcement learning models and algorithms.
That is not to say there is a lack of interaction; RL is being used in some experimental physics domains, but physics has not yet as directly informed RL as it has, e.g., graphical models \citep{MLPhysicalScience}.

Nevertheless, we should expect deep connections between reinforcement learning and physics: an RL agent is trying to find a policy that maximizes expected reward and many natural phenomena can be viewed through a minimization principle.
For example, in classical mechanics or electrodynamics, a mass or light will follow a path that minimizes a physical quantity called the \emph{action}, a property known as the \emph{principle of least action}. 
Similarly, in thermodynamics, a system with many degrees a freedom---such as a gas---will explore its configuration space in the search for a configuration that minimizes its free energy.
In reinforcement learning, rewards and value functions have a very similar flavor to energies, as they are extensive quantities and the agent is trying to find a path that maximizes them.
In RL, however, value functions are often treated as the central object of study.
This stands in contrast to statistical physics formulations of such problems in which the \emph{partition function} is the primary abstraction, from which all the relevant thermodynamic quantities---average energy, entropy, heat capacity---can be derived.
It is natural to ask, then, \emph{is there a theoretical framework for reinforcement learning that is centered on a partition function, in which value functions can be interpreted via average energies?}

In this work, we show how to construct a partition function for a reinforcement learning problem.
In a deterministic environment (Section~\ref{sec:deterministic}), the construction is elementary and very natural.
We explicitly identify the link between the underlying average energies associated with these partition functions and value functions of Boltzmann-like stochastic policies.
As in the inference-based view on RL, moving from deterministic to stochastic environments introduces complications.
In Section \ref{sec:stochastic}, we propose a construction for stochastic environments that results in realistic policies.
Finally, in Section \ref{sec:modelFree}, we show how the partition function approach leads to an alternative model-free reinforcement learning algorithm that does not explicitly represent value functions.

We model the agent's sequential decision-making task as a Markov decision process (MDP), as is typical.
The agent selects actions in order to maximize its cumulative expected reward until a final state is reached. 
The MDP is defined by the objects~$(\mcS,\mcA,\mcR,\mcP)$.
$\mcS$ and~$\mcA$ are the sets of states and actions, respectively.
${\mcP(s,a,s') = \mathbb{P}(s'\mid s,a)}$ is the probability of landing in state~$s'$ after taking action~$a$ from state~$s$.
$\mcR(s,a,s')$ is the reward resulting from this transition. 
We also make the following additional assumptions: 1)~$\mcS$ is finite, 2) all rewards~$\mcR(s,a,s')$ are bounded from above by~$\mcR_{\text{max}}$ and deterministic, and 3) the number of available actions is uniformly bounded over all states by~$d$.
We also allow for terminal states to have rewards  even though there are no further actions and transitions. We denote these final-state rewards by~$\mcR(s_f)$. By shifting all rewards by~$\mcR_{\text{max}}$ we can assume without loss of generality that~${\mcR_{\text{max}}=0}$ making all transition  rewards~$\mcR(s,a,s')$ non positive. The final state rewards~$\mcR(s_f)$ are still allowed to be positive however.

\section{Partition Functions for Deterministic MDPs} \label{sec:deterministic}

Our starting point is to consider deterministic Markov decision processes.
Deterministic MDPs are those in which the transition probability distributions assign all their mass to one state.
Deterministic MDPs are a widely studied special case \citep{madani2002polynomial, wen2013efficient,  dekel2013better} and they are realistic for many practical control problems, such as robotic manipulation and locomotion, drone maneuver or machine-controlled scientific experimentation.
For the deterministic setting, we will use~${s+a}$ to denote the state that follows the taking of action~$a$ in state~$s$.
Similarly, we will denote the reward more concisely as~$\mcR(s,a)$.

\subsection{Construction of State-Dependent Partition Functions}
\label{sec:deterministicConstruction}
To construct a partition function, two ingredients are needed: a statistical ensemble, and an energy function~$E$ on that ensemble.
We will construct our ensembles from trajectories through the MDP; a trajectory~$\omega$ is a sequence of tuples~${\omega = (s_0, a_0, r_0), (s_1, a_1, r_1),\ldots,(s_T,a_T,r_T)}$ such that state~$s_{T+1}$ is a terminal state.
We use the notation~$s_t(\omega)$,~$a_t(\omega)$, and~$r_t(\omega)$ to indicate the state, action, and reward, respectively, of trajectory~$\omega$ at step~$t$.
Each state-dependent ensemble~$\Omega(s)$ is then the set of all trajectories that start at~$s$, i.e., for which~${s_0(\omega) = s}$.
We will use these ensembles to construct a partition function for each state~${s \in \mcS}$.
Taking~$|\omega|$ to be the length of the trajectory, we write the energy function as
\begin{align}
E(\omega) &= -\sum_{t=0}^{|\omega|-1} r_t(\omega) - R(s_{|\omega|}) = -\sum_{t=0}^{|\omega|} r_t(\omega)\,.
\end{align}
The form on the right takes a notational shortcut of defining~$r_{|\omega|}(\omega):=R(s_{T+1})$ for the reward of the terminal state.
Since the agent is trying to maximize their cumulative reward,~$E(\omega)$ is a reasonable measure of the agent's preference for a trajectory in the sense that lower energy solutions accumulate higher rewards.
Note in particular that the ground state configurations are the most rewarding trajectories for the agent.
With the ingredients~$\Omega(s)$ and~$E(\omega)$ defined, we get the following partition function
\begin{align}
\mcZ(s,\beta) = \sum_{\omega \in \Omega(s)}e^{-\beta \,E(\omega)}=\sum_{\omega \in \Omega(s)} e^{\beta \sum_{t=0}^{|\omega|} r_t(\omega)}\,.
\end{align}
In this expression,~$\beta \geq 0~$ is a hyper-parameter that can be interpreted as the inverse of a temperature.
(This interpretation comes from statistical physics where~${\beta = \frac{1}{K_B T }}$, where~$K_B$ is the Boltzmann constant.)
This partition function does not distinguish between two trajectories having identical cumulative rewards but different lengths.
However, among equivalently rewarding trajectories, it seems natural to prefer shorter trajectories.
One way to encode this preference is to add an explicit penalty~${\mu\leq 0}$ on the length~$|\omega|$ of a trajectory, leading to a partition function
\begin{align}\label{eq:partition}
\mcZ(s,\beta) = \sum_{\omega \in \Omega(s)} e^{-\beta \, E(\omega) + \mu |\omega| }\,.
\end{align}
In statistical physics,~$\mu$ is called a \emph{chemical potential} and it measures the tendency of a system (such as a gas) to accept new particles.
It is sometimes inconvenient to reason about systems with a  fixed number of particles, adding a chemical potential offers a way to relax that constraint, allowing a system to have a varying number of particles while keeping the average fixed.

Note that since MDPs can allow for both infinitely long trajectories and infinite sets of finite trajectories,~$\Omega(s)$ can be infinite even in relatively simple settings.
In Appendix~\ref{proof:wellDefined}, we find that a sufficient condition for~$\mcZ(s,\beta)$ to be well defined is taking~${\mu < -\log{d}}$.
As written, the partition function in Eq.~\ref{eq:partition} is ambiguous for final states.
For clarity we define~${\mcZ(s_f,\beta) := e^{\beta \, R(s_f)}}$ for a terminal state~$s_f$.
We will refer to these as the boundary conditions.

Mathematically, the parameter~$\mu$ has a similar role as the one played by~$\gamma$, the discount rate commonly used in reinforcement learning problems.
They both make infinite series convergent in an infinite horizon setting, and ensure that the Bellman operators 
are contractions in their respective frameworks~(\ref{proof:contraction} ,\ref{proof:rhoContraction}).
However, when using~$\gamma$, the order in which the rewards are observed can have an impact on the learned policy which does not happen when~$\mu$ is used.
This could be a desirable property for some problems as it uncouples rewards from preferences for shorter paths.

\subsection{A Bellman Equation for~$\mcZ$}
\label{sec:deterministicBellman}

As we have defined an ensemble~$\Omega(s)$ for each state~${s\in\mcS}$, there is a partition function~$\mcZ(s,\beta)$ defined for each state.
These partition functions are all related through a Bellman-like recursion:
\begin{align}\label{eq:z-bellman}
\mcZ(s, \beta) &= \sum_{a } e^{\beta\, \mcR(s,a) +\mu} \;\mcZ(s+a, \beta)\,,
\end{align}
where, as before,~${s+a}$ indicates the state deterministically following from taking action~$a$ in state~$s$.
This Bellman equation can be easily derived by decomposing each trajectory~${\omega \in \Omega(s)}$ into two parts: the first transition resulting from taking initial action~$a$ and the remainder of the trajectory~$\omega'$ which is a member of~$\Omega(s+a)$.
The total energy and length can also be decomposed in the same way, so that:
\begin{align*}
\mcZ(s, \beta) &= \sum_{\omega \in \Omega(s)} e^{-\beta \,E(\omega) + \mu |\omega|}
= \sum_{\omega \in \Omega(s)} e^{\beta\sum_{t=0}^{|\omega|}  r_{t}(\omega) + \mu |\omega|}\\
&=  \sum_{a\in\mcA} e^{\beta\,\mcR(s,a) +\mu}
\sum_{\omega' \in \Omega(s+a) }  e^{\beta \sum_{t=1}^{|\omega|}  r_{t}(\omega) + \mu( |\omega|-1)}
= \sum_{a\in\mcA} e^{\beta\,\mcR(s,a) +\mu}
\sum_{\omega' \in \Omega(s+a) }e^{-\beta \,E(\omega') + \mu|\omega'|}\\
&=  \sum_{a } e^{\beta \,\mcR(s,a) +\mu} \;\mcZ(s+a, \beta)\,.
\end{align*}
Note in particular that this Bellman recursion is \textbf{linear} in~$\mcZ$.

\subsection{The Underlying Value Function and Policy}\label{sec:deterministicPolicy}
The partition function can be used to compute an average energy to shed light on the behavior of the system.
This average is computed under the Boltzmann (Gibbs) distribution induced by the energy on the ensemble of trajectories :
\begin{align}
\bbP(\omega \given \beta, \mu, s_0(\omega)=s) &= \frac{\indicator_{\Omega(s)}(\omega)}{\mcZ(s,\beta)}e^{-\beta\,E(\omega) + \mu|\omega|}\,.
\end{align}
In probabilistic machine learning, this is usually how one sees the partition function: as the normalizer for an energy-based learning model or an undirected graphical model (see, e.g., \citet{murray2004bayesian}). Under this probability distribution, high-reward trajectories are the most likely but sub-optimal ones could still be sampled. This approach is closely related to the {\it soft-optimality} approach to RL \citep{Levine18}.
This distribution over trajectories allows us to compute an average energy for state~$s$ either as an explicit expectation or as the partial derivative of the log partition function with respect to the inverse temperature:
\begin{align}
\langle E \rangle &= \sum_{\omega\in\Omega(s)}\frac{1}{\mcZ(s,\beta)}e^{-\beta\,E(\omega) + \mu|\omega|}
E(\omega) = -\frac{\partial}{\partial \beta}\log \mcZ(s,\beta)\,.
\end{align}
The negative of the average energy is the value function:~$V(s,\beta) := -\langle E \rangle = \frac{\partial}{\partial \beta}\log \mcZ(s,\beta)$.
This is an intuitive result: recall that the energy~$E(\omega)$ is low when the trajectory~$\omega$ accumulates greater rewards, so lower average energy indicates that the expected cumulative reward---the value---is greater.
Since the partition functions~$\{\mcZ(s,\beta)\}_{s \in S}$ are connected by a Bellman equation, we expect that the underlying value functions~$\{V(s,\beta)\}_{s \in S}$ would be connected in a similar way, and there is indeed a non-linear Bellman recursion:
\begin{align*}
V(s,\beta) &= \frac{\partial}{\partial\beta}\log\mcZ(s,\beta)
= \frac{1}{\mcZ(s,\beta)}\frac{\partial}{\partial\beta}\mcZ(s,\beta)
= \frac{1}{\mcZ(s,\beta)} \frac{\partial}{\partial \beta}  \sum_{a\in\mcA} e^{\beta \,\mcR(s,a) +\mu} \;\mcZ(s+a, \beta)\\
&= \frac{1}{\mcZ(s,\beta)} \sum_{a\in\mcA} e^{\beta \,\mcR(s,a) +\mu} \frac{\partial}{\partial\beta}\mcZ(s+a, \beta)+\mcR(s,a)e^{\beta \,\mcR(s,a) +\mu} \;\mcZ(s+a, \beta) \,.
\end{align*}
The derivative rule for natural log gives us~$\frac{\partial}{\partial\beta}\mcZ(s,\beta)=\mcZ(s,\beta)\frac{\partial}{\partial\beta}\log\mcZ(s,\beta) = \mcZ(s,\beta) V(s,\beta)$, so:
\begin{align}
V(s,\beta) &= \frac{1}{\mcZ(s,\beta)} \sum_{a\in\mcA} e^{\beta \,\mcR(s,a) +\mu}\mcZ(s+a,\beta) V(s+a,\beta) +\mcR(s,a)e^{\beta \,\mcR(s,a) +\mu} \;\mcZ(s+a, \beta)\notag\\
&= \frac{1}{\mcZ(s,\beta)} \sum_{a\in\mcA} e^{\beta \,\mcR(s,a) +\mu}\mcZ(s+a,\beta)\left[
V(s+a,\beta) + \mcR(s,a)
\right]\,.\label{eq:value-bellman}
\end{align}
Note that the quantities~$e^{\beta \,\mcR(s,a) +\mu}\mcZ(s+a,\beta)$ inside the summation of Eq.~\ref{eq:value-bellman} are positive and sum to~$\mcZ(s,\beta)$ due to the Bellman recursion for~$\mcZ(s,\beta)$ from Eq.~\ref{eq:z-bellman}.
Thus we can view this Bellman equation for~$V(s,\beta)$ as an expectation under a distribution on actions, i.e., a \emph{policy}:
\begin{align}
	V(s,\beta) &= \sum_{a\in\mcA} \pi(a \given s)\left[
V(s+a,\beta) + \mcR(s,a)
\right] & \pi(a \given s) &= \frac{\mcZ(s+a,\beta)}{\mcZ(s,\beta)}e^{\beta \,\mcR(s,a) +\mu}\,.
\end{align}
The policy~$\pi$ resembles a Boltzmann policy but strictly speaking it is not. A Boltzmann policy~$\pi_B$ selects actions proportionally to the exponential of their expected cumulative reward:~${\pi_{\text{B}}(a \mid s) \propto \exp\left(\beta \left[\mcR(s,a)+V(s+a) \right]\right)}$.  In particular,~$\pi_B$ does not take {\it entropy} into account: if two actions have the same expected optimal value, they will be picked with equal probability regardless of the possibility that one of them could achieve this optimality in a larger number of ways. In the partition function view,~$\pi$ does take entropy into account and to clarify this difference we will look at the two extreme cases~${\beta \to \{0,\infty\}}$.
\clearpage
\noindent When~${\beta \to 0}$, where the temperature of the system is infinite, rewards become irrelevant and we find that:~${\pi(a \mid s) \propto \sum_{\omega \in \Omega(s+a) } e^{\mu |\omega|}}$.  This means that~$\pi$ is picking action~$a$ proportionally to the number of trajectories that begin with~${s+a}$.  Here the counting of trajectories happens in a weighted way: longer trajectories contribute less than shorter ones. This is different from a Boltzmann policy that would pick actions uniformly at random.
\medbreak
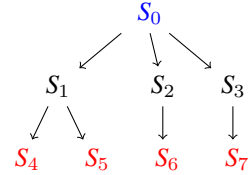
\begin{wrapfigure}{r}{0.35\textwidth}%
\vspace{-0.5cm}%
{\normalsize%
\begin{center}%
\begin{forest}%
sn edges%
[{\color{blue}~$S_0$}%
	[$S_1$%
		[{\color{red}~$S_4$}%
		]%
		[{\color{red}~$S_5$}%
		]%
	]%
	[$S_2$%
		[{\color{red}~$S_6$}%
		]%
	]%
	[$S_3$%
		[{\color{red}~$S_7$}%
		]%
	]%
]%
\end{forest}%
\end{center}%
}%
\vspace{-0.5cm}%
\captionof{figure}{Decision Tree MDP}\label{fig:tree}%
\vspace{-0.5cm}%
\end{wrapfigure}%

\noindent When~${\beta \to \infty}$, the low-temperature limit, we find in Section~\ref{proof:Boltzmann}  that~${\pi(a \mid s) \propto  N_{\max}(s+a) \exp\left(\beta \left[\mcR(s,a)+V(s+a) \right]\right)}$  where ${N_{\max}(s+a)}$ is a weighted count of the number of {\bf optimal} trajectories that begin at the state~${s+a}$. Boltzmann policies completely  ignore the~$N_{\max}$ entropic factor.
\medbreak
To illustrate this difference more clearly, we consider the  deterministic decision tree MDP shown in Figure~\ref{fig:tree} where~$S_0$ is the initial state and the leafs~$S_4$,~$S_5$,~$S_6$, and~$S_7$ are the final states. 
The arrows represent the actions available at each state.
There are no rewards and the boundary conditions are:~${\mcR(S_{4}) = \mcR(S_5) = \mcR(S_6) =1}$ and~${\mcR(S_7) = 0}$.
This gives us the boundary condition:~${\mcZ(S_{4},\beta) = \mcZ(S_5,\beta) = \mcZ(S_6,\beta) =e^{\beta}}$ and~${\mcZ(S_7,\beta) = 1}$. Computing the $\mcZ$-functions at the intermediate states~$S_1, S_2$ and~$S_3$ we find: ~${\mcZ(S_1,\beta) = 2e^{\beta +\mu}}$,~${\mcZ(S_2,\beta) = e^{\beta +\mu}}$ and~${\mcZ(S_3,\beta) = e^{\mu}}$.  Finally we have~${\mcZ(S_0,\beta) = 3e^{\beta+2\mu}+e^{2\mu}}$.  
The underlying policy for picking the first action is given by:
\begin{align}
\pi_{\beta}(1\mid 0) =\frac{2e^{\beta +2\mu}}{3e^{\beta+2\mu}+e^{2\mu}} ~~~~~ \pi_{\beta}(2\mid 0) =\frac{e^{\beta +2\mu}}{3e^{\beta+2\mu}+e^{2\mu}} ~~~~~ \pi_{\beta}(3\mid 0) =\frac{e^{2\mu}}{3e^{\beta+2\mu}+e^{2\mu}}
\end{align}
When~$\beta \to 0$,  we get:  ~$\pi_0(1\mid 0) = \frac{1}{2}, ~~ \pi_0(2 \mid 0) = \frac{1}{4}, ~~ \pi_0(3 \mid 0)  = \frac{1}{4}$.
A Boltzmann policy would pick these three actions with equal probability.  The policy~$\pi$ is biased towards the heavier subtree.\\
When~$\beta \to \infty$ we get: ~$\pi_{\infty}(1\mid 0) = \frac{2}{3}, ~~ \pi_{\infty}(2 \mid 0) = \frac{1}{3}, ~~ \pi_{\infty}(3 \mid 0)=0$. A Boltzmann policy would pick action~$1$ and~$2$ with a probability of~$\frac{1}{2}$. ~$\pi$ prefers states from which many possible optimal trajectories are possible.

\subsection{A Planning Algorithm}

When the dynamics of the environment are known, it is possible to to learn~$~\mcZ(s,\beta)$ by exploiting the Bellman equation (\ref{eq:z-bellman}). 
We denote by~${s \to s'}$ the property that there exists an action~$a$ that takes an agent from state~$s$ to state~$s'$. The reward associated with this transition will be denoted~$\mcR({s \to s'})$.  Let~${\mcZ(\beta) = [\mcZ(s,\beta)]_{s\in \mcS }}$ be the vector of all partition functions and~${C(\beta) \in \mathbb{R}^{|\mcS| \times |\mcS|}}$ be the matrix:
\begin{align}C(\beta)_{s,s'} =  \indicator_{s \to s'}e^{\beta \mcR({s \to s'})+\mu} + \indicator_{s = s' = \text{final state}}
\end{align}
$C(\beta)$ is a matrix representation of the Bellman operator in Eq.~\ref{eq:z-bellman}. With these notations, the Bellman equations in (\ref{eq:z-bellman}) can be compactly written as:~${\mcZ(\beta) =   C(\beta)\,  \mcZ(\beta)}$ highlighting the fact that~$\mcZ(\beta)$ is a fixed point of the map:~${\phi: X \to C(\beta)  X}$. In Appendix~\ref{proof:contraction}, we show that~$\phi$ is a contraction which makes it possible to learn~$\mcZ(\beta)$  by starting with an initial vector~$\mcZ_0$ having compatible boundary conditions and successively iterating the map~$\phi$:~$\mcZ_{n+1} = C(\beta) \mcZ_{n}$. We could also interpret ~$\mcZ(\beta)$ as an eigenvector of~$C(\beta)$. In this context, this algorithm is simply doing a power method.

Interestingly, we can learn~$\mcZ(\beta)$ by solving the underdetermined linear system~${[I_{|\mcS|}- C(\beta)] \,~\mcZ(\beta) = 0_{|\mcS|}}$ with the right boundary conditions.
We show in Appendix~\ref{proof:Boltzmann} that the policies learned are related to Boltzmann policies which produce non linear Bellman equations at the value function level: 
 \begin{align}
 V(s, \beta) =  \sum_{a } \frac{e^{\beta \left(\mcR(s,a) + \gamma V(s+a, \beta) \right)}}{\mcW(s,\beta)}   [r_{(s,a)}  +   \gamma V(s+a,\beta )]
 \end{align}
where~$\gamma$ is the discount factor and~$\mcW(s,\beta) =\sum_{a } e^{\beta \left(\mcR(s,a) + \gamma V(s+a, \beta) \right)}$ is a normalization constant different from~$\mcZ(s,\beta)$.
By working with partition functions we transformed a non linear problem into a linear one. This remarkable result is reminiscent of linearly solvable MDPs \citep{todorov2007linearly}.\\\\
Once~$\mcZ$ is learned the agent's policy is given by:~${\,\mathbb{P}(a \mid s) \propto ~e^{\beta\mcR(s,a)}\mcZ(s+a,\beta)}$.

\section{Partition functions for Stochastic MDPs}
We now move to the more general MDP setting.  The dynamics of the environment can now be stochastic. However, as mentioned at the end of the introduction, we still assume that given an initial state~$s$, an action~$a$, and a landing state~$s'$, the reward~$\mcR(s,a,s')$ is deterministic.
\subsection{A First Attempt: Averaging the Bellman Equation}
\label{sec:unrealisticPolicy}
A first approach to incorporating the stochasticity of the environment is to average the right-hand side of the Bellman equation (\ref{eq:z-bellman}) and define~$\mcZ(s,\beta)$ as the solution of:
\begin{equation}
\label{eq:avgBellman}
\mcZ(s,\beta) = \sum_{a} \mathbb{E}_{s' \mid s,a} \left[e^{\beta \mcR(s,a,s') +\mu }~\mcZ(s',\beta)\right] = \sum_{a,s'} \mathbb{P}(s'\mid s,a) \, e^{\beta \mcR(s,a,s') +\mu }~\mcZ(s',\beta)\,.
\end{equation}
Interestingly, the solution of this equation can be constructed in the same spirit of Section~\ref{sec:deterministicConstruction} by summing a functional over the set of trajectories. If we define~$L(\omega)$ to be the log likelihood of a trajectory:~${L(\omega) =  \sum_{t=0}^{|\omega|-1} \log{\mathbb{P}(s_{t+1}\mid{s_t,a_t})}}$ then~$\mcZ(s,\beta)$ is defined by
\begin{align}
\label{eq:unrealisticZ}
\mcZ(s,\beta) = \sum_{\omega \in \Omega(s)} e^{-\beta E(\omega)+ \mu |\omega| + L(\omega) }\,,
\end{align}
satisfies the Bellman equation (\ref{eq:avgBellman}). The proof can be found in Appendix~\ref{proof:avgBellman}. In Appendix~\ref{proof:unrelasticBellman} we derive the Bellman equation satisfied by the underlying value function~$V(s,\beta)$ and we find: 
\begin{align}
\label{eqn:unrealisticPolicy}
V(s,\beta)&= \sum_{a,s'} \frac{ e^{\beta  \mcR(s,a,s')+\mu} \,\,~\mcZ(s',\beta)}{\mcZ(s,\beta)}\times \mathbb{P}(s'\mid s,a) \times \left( \mcR(s,a,s') +V(s',\beta) \right) \,.
\end{align}
This Bellman equation does not correspond to a realistic policy; the policy depends on the landing state~$s'$ which is a random variable. The agent's policy and the environment's transitions cannot be decoupled. This is not surprising, from Eq.~\ref{eq:unrealisticZ} we see that~$\mcZ$ puts rewards and transition probabilities on an equal footing. As a result an agent believes they can choose any available transition as long as they are willing to pay the price in log probability. This encourages risky behavior: the agent is encouraged to bet on highly unlikely but beneficial transitions. These observations were also noted in \citet{Levine18}. 

\subsection{A Variational Approach}\label{sec:stochastic}
Constructing a partition function for a stochastic MDP is not straightforward because there are two types of randomness: the first comes from the agent's policy and the second from stochasticity of the environment.  Mixing these two sources of randomness can lead to unrealistic policies as we saw in Section \ref{sec:unrealisticPolicy}. A more principled approach is needed.

We construct a new deterministic MDP~$(\tilde{\mcS},\tilde{\mcA},\tilde{\mcR},\tilde{\mcP})$ from~$(\mcS,\mcA,\mcR,\mcP)$. We take~$\tilde{\mcS}$ to be the space of probability distributions over~$\mcS$, similar to belief state representations for partially-observable MDPs \citep{astrom1965optimal,sondik1978optimal,kaelbling1998planning}.  We make the assumption that the actions~$\mcA$ are the same for all states and take~${\tilde{\mcA} = \mcA}$.
For~${\rho \in \tilde{\mcS}}$ and~${a \in \tilde{\mcA}}$ we  define~${\tilde{\mcP}(\rho,a) := {P_a}^T \rho}$ where~$P_a$ is the transition matrix corresponding to choosing action~$a$ in the original MDP.  We define~${\tilde{\mcR}(\rho,a) := \E_{s \sim \rho} \left[\E_{ s' \mid s,a} [\mcR(s,a,s')] \right]}$. 

$\mcS$ being finite, it has a finite number~$M$ of final states which we denote~$\{{f_i}\}_{i\in\{1,\cdots,M\}}$. The final states of~$\tilde{\mcS}$  are of the form~${\rho_f = \sum_{i=1}^M \alpha_i \delta_{{f_i}}}$ where~${0\leq \alpha_i \leq 1}$ verify~${\sum_{i=1}^M \alpha_i =1}$ and~$\delta_{f_i}$ is a Dirac delta function at state~$f_i$. The intrinsic value~$\rho_f$ of such a final state is then given by~${\mcR(\rho_f) = \sum_{i=1}^M \alpha_i \mcR({f_i})}$. This leads to the  boundary conditions: 
\begin{align}
\label{eqn:rhoBoundary}
\mcZ(\rho_f) = \exp \left(\beta \sum_{i=1}^M \alpha_i \mcR({s_{f_i}})\right) = \prod_{i=1}^M \mcZ({f_i},\beta)^{\alpha_i} \,.
\end{align}
This new MDP~$(\tilde{\mcS},\tilde{\mcA},\tilde{\mcR},\tilde{\mcP})$ is deterministic, and we can follow the same approach of Section~\ref{sec:deterministic} and construct a partition function ~$\mcZ(\rho,\beta)$ on~$\tilde{\mcS}$. ~$\mcZ(s,\beta)$ can be recovered by evaluating~$\mcZ(\delta_{s},\beta)$. From this construction we also get that~$\mcZ(\rho,\beta)$ satisfies the following Bellman equation:
\begin{align}
\label{eqn:bellmanRho}
\mcZ(\rho,\beta) = \sum_{a} e^{\beta \mcR(\rho,a) +\mu } ~\mcZ({P_{a}}^T \rho,\beta)\,.
\end{align}
Just as it is the case for deterministic MDPs, the Bellman operator associated with this equation is a contraction. This is proved in Appendix~\ref{proof:rhoContraction}. However~$\tilde{\mcS}$ is now infinite which makes solving Eq.~\ref{eqn:bellmanRho} intractable.  We adopt a variational approach which consists in finding the best approximation of~$\mcZ(\rho,\beta)$ within a parametric family~$\{\mcZ_{\theta}\}_{\theta \in \Theta}$. We measure the fitness of a candidate through the following loss function:~${\Delta(\theta) = \frac{1}{|\mcS|} \sum_{s \in \mcS} \left(\mcZ_{\theta}(\delta_{s},\beta) -\sum_{a} e^{\beta \mcR(\delta_{s},a) +\mu } ~\mcZ_{\theta}({P_{a}}^T \delta_{s},\beta) \right)^2 }$.

For illustration purposes, and inspired by the form of the boundary conditions (\ref{eqn:rhoBoundary}), we  consider a simple parametric family given by  the partition functions of the form~${\mcZ_{\theta}(\rho) = \prod_{i=1}^{|\mcS|} {\theta_{i}}^{\rho_i}}$, where~${\theta \in \mathbb{R}^{|\mcS|}}$. The optimal~$\theta$ can be found using usual optimization techniques such as gradient descent. By evaluation of~$\mcZ_{\theta}$ at~${\rho = \delta_{S_i}}$ we see that we must have~${\theta_i = \mcZ(\delta_{S_i})= \mcZ(S_i)}$ and consequently we have~${\mcZ_{\theta}(\rho) = \prod_{i=1}^{|\mcS|} {\mcZ(S_i)}^{\rho_i}}$. The optimal solution satisfies the following Bellman equation:
\begin{align} 
\label{eqn:variationalBellman}
\mcZ(s,\beta) \approx \sum_{a}\prod_{s' \in \mcS}\left[ e^{\beta \mcR(s,a,s') +\mu } ~ \mcZ(s',\beta)\right]^{\mathbb{P}(s' \mid s,a)}
\end{align}
The underlying value function verifies~${V(s,\beta) \approx \sum_{a,s'} \pi(a \mid s)\, \mathbb{P}(s' \mid s,a) \left(\mcR(s,a,s') + V(s',\beta) \right)}~$ where the policy~$\pi$ is given by~${\pi(a \mid s) \propto\prod_{s' \in \mcS} \left[e^{\beta  \mcR(s,a,s')+\mu} \,~\mcZ(s',\beta)\right]^{\mathbb{P}(s' \mid s,a)}}$. This approach leads to a realistic policy as its only dependency is on the current state, not a future one, unlike the policies arising from Eq.~\ref{eqn:unrealisticPolicy}.

\section{The Model-Free Case}
\label{sec:modelFree}
\subsection{Construction of  State-Action-Dependent Partition Function}
In a model free setting, where the transition dynamics are unknown, state-only value functions such as~$V(s)$ are less useful than state-action value functions such as~$Q(s,a)$. Consequently, we will extend our construction to state-action partition functions~$\mcZ(s,a,\beta)$. For a deterministic environment, we extend the construction in Section \ref{sec:deterministic} and define~$\mcZ(s,a,\beta)$ by
\begin{align}
\mcZ(s,a,\beta) =  \sum_{\omega \in \Omega(s,a)} e^{-\beta E(\omega)+\mu|\omega|} =\sum_{\omega \in \Omega(s,a)} e^{\beta \sum_{i=0}^{|\omega|}r_i+\mu|\omega|}
\end{align}
where~$\Omega(s,a)$ denotes the set of trajectories having~${(s_0,a_0)=(s,a)}$. Since~${\Omega(s) = \bigcup_{a\in \mcA} \Omega(s,a)}$, we have
${\mcZ(s,\beta) = \sum_{a} \mcZ(s,a,\beta)}$.  As a consequence of this construction,~$\mcZ(s,a,\beta)$ satisfies the following linear Bellman equation:
\begin{align}
\mcZ(s,a,\beta) = e^{\beta \mcR(s,a) +\mu}\sum_{a'} \mcZ(s+a,a',\beta)\,.
\end{align}
This Bellman equation can be easily derived by decomposing each trajectory~${\omega \in \Omega(s,a)}$ into two parts: the first transition resulting from taking initial action~$a$ and the remainder of the trajectory~$\omega'$ which is a member of~$\Omega(s+a,a')$ for some action~$a' \in \mcA$ .
The total energy and length can also be decomposed in the same way, so that:
\begin{align*}
\mcZ(s,a,\beta) &=  \sum_{\omega \in \Omega(s,a)} e^{-\beta E(\omega)+\mu|\omega|} =\sum_{\omega \in \Omega(s,a)} e^{\beta \sum_{i=0}^{|\omega|}r_i+\mu|\omega|}\\
&=  e^{\beta\,\mcR(s,a) +\mu}
\sum_{\omega \in \Omega(s,a) }  e^{\beta \sum_{t=1}^{|\omega|}  r_{t}(\omega) + \mu (|\omega|-1)}
= e^{\beta\,\mcR(s,a) +\mu}
\sum_{\omega' \in \Omega(s+a) }e^{-\beta \,E(\omega') + \mu|\omega'|}\\
&= e^{\beta\,\mcR(s,a) +\mu}
\sum_{a' \in \mcA } \sum_{\omega' \in \Omega(s+a,a') } e^{-\beta \,E(\omega') + \mu|\omega'|} =   e^{\beta\,\mcR(s,a) +\mu}
\sum_{a' \in \mcA } \mcZ(s+a,a',\beta).
\end{align*}
\noindent In the same spirit of Section \ref{sec:deterministicPolicy}, one can show that the average underlying value function ${Q(s,a,\beta)= \frac{\partial}{\partial \beta} \log{\mcZ(s,a,\beta)}}$ satisfies a Bellman equation: 
\begin{align}
Q(s,a,\beta) = \mcR(s,a)+ \sum_{a'} \pi(a' \mid s+a)~Q(s+a,a',\beta) ~~~~~~~~ \pi(a \mid s ) = \frac{\mcZ(s,a,\beta)}{\sum_{a'} \mcZ(s,a',\beta)}
\end{align} 
$Q(s,a,\beta)$ can be then reinterpreted as the~$Q$-function of the policy~$\pi$. Similarily to the results of Section~\ref{sec:deterministicPolicy} and Appendix~\ref{proof:Boltzmann}, the policy~$\pi$ can be thought of a Boltzmann policy of parameter~$\beta$ that takes entropy into account. 
This construction can be extend to a stochastic environments by following the same approach used in Section~\ref{sec:stochastic}.\\

In the following we show how learning the state-action partition function~$\mcZ(s,a,\beta)$ leads to an alternative approach to model-free reinforcement learning  that does not explicitly represent value functions. 

\subsection{A Learning Algorithm}

In~$Q$-Learning, the update rule typically consists of a linear interpolation between the current value estimate and the one arising \emph{a posteriori}:\begin{align}
Q(s_t,a_t) \leftarrow (1-\alpha)Q(s_t,a_t)+\alpha\left(r_t+\gamma \max_{a_{t+1}}Q(s_{t+1},a_{t+1})\right)
\end{align}
 where~${\alpha\in [0,1]}$ is the learning rate and~$\gamma$ is the discount factor.  For~$\mcZ$-functions we will replace the linear interpolation with a geometric one. We take the update rule for~$\mcZ$-functions to be the following:
\begin{align}
\label{eqn:updateRule}
\mcZ(s_t,a_t,\beta) \leftarrow \mcZ(s_t,a_t,\beta)^{1-\alpha} \times \left(e^{\beta r_t +\mu}\sum_{a_{t+1}} \mcZ(s_{t+1},a_{t+1},\beta) \right)^{\alpha}\,.
\end{align}
To understand what this update rule is doing, it is insightful to look at how how the underlying ~$Q$-function, ~~${Q(s,a)=\frac{\partial}{\partial \beta} \log{\mcZ(s_t,a_t,\beta)}}$ is updated. We find:
\begin{align}
\label{eqn:QSARSA}
Q(s_t,a_t,\beta) \leftarrow (1-\alpha) Q(s_t,a_t,\beta) + \alpha \left(r_t+ \sum_{a_{t+1}}\frac{\mcZ(s_{t+1},a_{t+1},\beta)}{\sum_{a'} \mcZ(s_{t+1},a',\beta)}Q(s_{t+1},a_{t+1},\beta) \right)\,.
\end{align}
 We see that we recover a weighted version of the SARSA update rule. This update rule is referred to as \emph{expected} SARSA.  Expected SARSA is known to reduce the variance in the updates by exploiting knowledge about stochasticity in the behavior policy  and hence is considered an improvement over vanilla SARSA \citep{expectedSarsa}.
 
Since the underlying update rule is equivalent to the expected SARSA update rule, we can use any exploration strategy that works for expected SARSA. One exploration strategy could be~$\epsilon$-greedy which consists in taking action~$a = \mbox{argmax}_{a \in \mcA} \mcZ(s,a,\beta)$ with probability~$1-\epsilon$ and picking an action uniformly at random with probability~$\epsilon$.  Another possibility would be a Boltzmann-like exploration which consists in taking action~$a$ with probability~$\mathbb{P}(a \mid s) \propto \mcZ(s,a,\beta)$.

We would like to emphasize that even though the expected SARSA update is not novel, the learned policies through this updates rule are proper to the partition-function approach. In particular, the learned policies~$\pi(a \mid s) \propto \mcZ(s,a,\beta)$ are Boltzmann-like policies with some entropic preference properties as described in Section~\ref{sec:deterministicPolicy} and Appendix~\ref{proof:Boltzmann}.

\section{Conclusion}
In this article we discussed how planning and reinforcement learning problems can be approached through the tools and abstractions of statistical physics.
We started by constructing partition functions for each state of a deterministic MDP and then showed how to extend that definition to the more general stochastic MDP setting through a variational approach.
Interestingly, these partition functions have their own Bellman equation making it possible to solve planning and model-free RL problems without explicit reference to value functions. 
Nevertheless, conventional value functions can be derived from our partition function and interpreted via  average energies.
Computing the implied value functions can also shed some light on the policies arising from these algorithms. 
We found that the learned policies are closely related to Boltzmann policies with the additional interesting feature that they take {\it entropy} into consideration by favoring states from which many trajectories are possible.
Finally, we observed that working with partition functions is more natural in some settings.
In a deterministic environment for example, near-optimal Bellman equations become linear which is not the case in a value-function-centric approach.  

\section{Acknowledgments}
We would like to thank Alex Beatson, Weinan E, Karthik Narasimhan and Geoffrey Roeder for helpful discussions and feedback.
This work was funded by a Princeton SEAS Innovation Grant and the Alfred P. Sloan Foundation.

\bibliography{references}
\bibliographystyle{plainnat}
\clearpage
\appendix
\section{Deterministic MDPs}
\subsection{$\mcZ(s,\beta)$ is well defined}\label{proof:wellDefined}
\begin{prop}
$\mcZ(s,\beta) = \sum_{\omega \in \Omega(s)} e^{\beta \sum_{i=0}^{|\omega|} r_i + \mu |\omega| }$ is well defined for~$\mu < -\log d$.
\end{prop}

\begin{proof}
The MDP being finite,~$\mcS$ has a finite number of final state we can then find a constant~$K$ such that, for all final states~$s_f$ we have~$\mcR(s_f) \leq K$.
\begin{align*}
\mcZ(s,\beta) =& \sum_{\omega \in \Omega(s)} e^{\beta \sum_{i=0}^{|\omega|} r_i + \mu |\omega| } \\ 
=& \sum_{\omega \in \Omega(s)} e^{\beta \sum_{i=0}^{|\omega|-1} r_i + \beta \mcR(s_{|\omega|}) +  \mu |\omega| } \\ 
\leq&\,\,  e^{\beta K} \sum_{\omega \in \Omega(s)} e^{\beta \sum_{i=0}^{|\omega|-1} r_i  +  \mu |\omega| } \\
\leq& \,\,  e^{\beta K}  \sum_{\omega \in \Omega(s)} e^{  \mu |\omega| } \\
\leq&  \,\,  e^{\beta K} \sum_{n\in \mathbb{N}} e^{  \mu n} \sum_{\omega \in \Omega(s), \, \abs{\omega}=n} 1\\
\leq&  \,\,  e^{\beta K} \sum_{n\in \mathbb{N}} e^{  \mu n }d^n\\
=&  \,\,  e^{\beta K} \sum_{n\in \mathbb{N}} (e^{ \mu+ \log d })^n
\end{align*}
Where used the fact that all rewards~$\{r_i \}_{i \in \{0,...,|\omega|-1\}}$ are non positive and that the number of available actions at each state is bounded by~$d$.
When~$\mu < -\log {d}$, the sum ~~$\sum_{n \in \mathbb{N}} (e^{ \mu +\log d })^n$ becomes convergent and ~$\mcZ(s,\beta)$ is well defined. 
\end{proof}
\begin{remark}
$\mu < -\log {d}$ is a sufficient condition, but not a necessary one.~$\mcZ(s,\beta)$ could be well defined for all values of~$\mu$. This happens for instance when~$\Omega(s)$ is finite for all~$s$. 
\end{remark}

\subsection{The underlying policy is Boltzmann-like}
\label{proof:Boltzmann}

For high values of~$\beta$, the sum~$\sum_{\omega \in \Omega(s)} e^{\beta \sum_{i=0}^{|\omega|} r_i + \mu |\omega| }$ will become dominated by the contribution of few of its terms. As~$\beta \to +\infty$, the sum will be dominated by the contribution of the paths with the biggest reward. We have
\begin{align*}
\log{\mcZ(s,\beta)} \underset{\beta \to \infty}{\sim} \beta ~ \mbox{max} \left\{\sum_{i=0}^{|\omega|} r_i(\omega),\, \omega \in \Omega(s) \right\} 
\end{align*}
We see that~$ V(s, \beta)= \frac{\partial}{\partial \beta} \log \mcZ(s,\beta) \underset{\beta \to \infty}{\rightarrow} \mbox{max} \left\{\sum_{i=0}^{|\omega|} r_i(\omega),\, \omega \in \Omega(s) \right\} ~$. 

Since the MDP is finite and deterministic, it has a finite number of transitions and rewards. Consequently, the set~$\left\{\sum_{i=0}^{|\omega|} r_i(\omega),\, \omega \in \Omega(s) \right\}$ takes  discrete values, in particular, there is a finite gap~$\Delta$ between the maximum value and the second biggest value of this set. Let's denote by~$\Omega_{\text{max}}(s)$ the set of trajectories that achieve this maximum and by~$N_{\max}(s) = \underset{{\omega \in \Omega_{\text{max}}(s)}}{\sum} e^{\mu |\omega|}~$.

$N_{\max}(s)$ counts the number of trajectories~$\Omega_{\text{max}}(s)$ in a weighted way: longer trajectories contribute less than shorter ones. It is a measure of the size of~$\Omega_{\text{max}}(s)$ that takes into account our preference for shorter trajectories.   Putting everything together we get:
\begin{align*}
\left(\frac{\mcZ(s,\beta)}{e^{\beta V(s,\beta)}} - N_{\max}(s)\right)   \underset{\beta \to \infty}{\leq} e^{-\beta \Delta} \sum_{\omega \in \Omega(s)} e^{  \mu |\omega| } \underset{\beta \to \infty}{\rightarrow}  0 
\end{align*}
This shows that ~$\mcZ(s,\beta)\underset{\beta \to \infty}{\sim} N_{\max}(s)~e^{\beta V(s,\beta)}$, which results in the following policy for~$\beta>>1$:
\begin{align*}
\pi(a \mid s) \underset{\beta \to \infty}{\propto }  N_{\max}(s+a) e^{\beta\left(\mcR(s,a) + V(s+a, \beta)\right)}
\end{align*}
$\pi$ differs from a traditional Boltzmann policy in the following way:  if we have two actions~$a_1$ and~$a_2$ such that ~$\mcR(s,a_1) + V(s+a_1, \beta) =\mcR(s,a_2) + V(s+a_2, \beta)$ but there are twice more optimal trajectories spanning from~$s+a_1$ than there are from~$s+a_2$ then action~$a_1$ will be chosen twice as often as~$a_2$. This is to contrast with the usual Boltzmann policy that will pick~$a_1$ and~$a_2$ with equal probability. When~$N_{\max}(s)$ is the same for all~$s$, we recover a Boltzmann policy.  When~$\beta \to +\infty$ the policy converges to a an optimal policy and~$V$ converges to the optimal value function.

\subsection{$X \to C(\beta)X$~ is a contraction}\label{proof:contraction}

\begin{prop}
Let  ~$\mcX(\beta) = \left\{Z \in \mathbb{R}^{|\mcS|}_{+} \text {  such for all final states }~ s_f  ~\text{we have } Z_{s_f}=e^{\beta \mcR(s_f)} \right\}$  and let ${C(\beta)_{s,s'} =  \indicator_{s \to s'}e^{\beta \mcR(s \to s')+\mu} + \indicator_{s = s' = \text{final state}}}$.  The map defined by  
\begin{equation*}
         \psi: \begin{cases}
              \mcX(\beta)  &\to \mcX(\beta)   \\
               X    &\to C(\beta) \, X\\
           \end{cases}
\end{equation*}
is a contraction for the sup-norm: ~$||x||_{\infty} = \underset{i \in \{1,\cdots,|\mcS|\}}{\max} \abs{x_i}$.
\end{prop}
\begin{proof}
$\mcX(\beta)$ is the set of all possible partition functions with compatible boundary conditions. The matrix~$C(\beta)$ is more explicitly defined by:
\begin{align*}
C(\beta)_{s,s'} = \begin{cases}1 & \mbox{if s = s'  and state s is a final state.} \\0 & \mbox{if there is no one step transition from state s to state s'.} \\e^{\beta \mcR(s \to s')+\mu} & \mbox{if the transition from state s to state s' generates reward }  \mcR(s \to s'). \end{cases}
\end{align*}
Because ~$C(\beta)_{s,s} =1$ when~$s$ is a final state, the map~$\psi$ is well defined (i.e.~$\mcX(\beta) \to \mcX(\beta)$).
Since the MDP is finite, it has a finite number of final state so there exists a constant~$K$ such that, for all final states~$s_f$ we have~$\mcR(s_f) \leq K$.\\\\
Let ~$X_1,~X_2 ~\in  \mcX(\beta)$ we have: 
\begin{align*}
\norm{ \psi(X_1)- \psi(X_2)}_{\infty} &=  \underset{i \in \{1,\cdots,|\mcS|\}}{\max} ~\abs{\left(C(\beta)X_1- C(\beta)X_2\right)_{i}}
\end{align*}
Without loss of generality we can assume that the MDP has~$m$ final states that are labeled~${|\mcS|-m+1,\cdots,|\mcS|}$.  Under this assumption we have:
\begin{align*}
\underset{i \in \{1,\cdots,|\mcS|\}}{\max} ~ \abs{(X_1- X_2)_{i}} = \underset{i \in \{1,\cdots,|\mcS|-m\}}{\max} ~\abs{(X_1- X_2)_{i}}
\end{align*}
This is because~$X_1$ and~$X_2$ have the same boundary conditions:~${\forall \, s_f \in  \{|\mcS|-m+1,\cdots,|\mcS|\}, ~ (X_1)_{s_f} = (X_2)_{s_f}}$. Since~$C(\beta)_{s_f,s_f}=1$ if~$s_f$ is the index a final state,~$C(\beta)X_1$ and~$C(\beta)X_2$ still have the same boundary conditions, we have:~${\forall  s_f \in  \{|\mcS|-m+1,\cdots,|\mcS|\}, ~ [C(\beta) X_1]_{s_f} = [C(\beta)X_2]_{s_f}}$. This gives us:

\begin{align*}
\underset{s \in \{1,\cdots,|\mcS|\}}{\max} ~ \abs{\left(C(\beta)X_1- C(\beta)X_2\right)_{s}}  =\underset{s \in \{1,\cdots,|\mcS|-m\}}{\max} ~\abs{\left( C(\beta)X_1- C(\beta)X_2\right)_{s}} 
\end{align*}
For~$s \in \left\{1,\cdots,|\mcS| \right\}$, we have:
$ \abs{\left( C(\beta)X_1- C(\beta)X_2\right)_{s}}  = \abs{\sum_{s'=1}^{|\mcS|} [C(\beta)]_{s,s'} ~ (X_1-X_2)_{s'}}$.\\\\
Since there are at most~$d$ available actions at each state and the environment is deterministic, at most ~$d$ coefficients~$C(\beta)_{s,s'}$ in this sum are non zero. Because the rewards are non positive, the non zero ones can be bounded by~$e^{\mu}$.\\\\
 Putting all these pieces together we can write:~$\abs{\sum_{s'=1}^{|\mcS|} [C(\beta)]_{s,s'} ~ (X_1-X_2)_{s'}} \leq d \times e^{\mu}  \norm{X_1-X_2}_{\infty}$.  Finally we get: 
 
 \begin{align*}
\norm{ C(\beta)X_1- C(\beta)X_2}_{\infty} \leq \underbrace{d \times e^{\mu}}_{< 1 ~\text{because } \mu < -\log d } \norm{ X_1-X_2}_{\infty}
\end{align*}
This proves that~$\psi$ is a contraction.
\end{proof}
\begin{remark}
We see here another mathematical similarity between the discount factor~$\gamma < 1$ usually used in RL and the chemical potential~$\mu <- \log d$. They both ensure that the Bellman operators are contractions.
\end{remark}

\section{Stochastic MDPs}

\subsection{Averaging the Bellman Equation and adding a likelihood cost are equivalent }\label{proof:avgBellman}

\begin{prop}
The partition function ~$\mcZ(s,\beta)$ defined by ~$\mcZ(s,\beta) := \sum_{\omega \in \Omega(s)} e^{-\beta E(\omega)+ \mu |\omega| + L(\omega) }~$ satisfies the following Bellman equation:
\begin{equation*}
\mcZ(s,\beta) = \sum_{a} \mathbb{E}_{s' \mid s,a} \left[e^{\beta \mcR(s,a,s') +\mu }~\mcZ(s',\beta)\right] 
\end{equation*}
\end{prop}

\begin{proof}

The proof follows the same path as the one in Section \ref{sec:deterministicBellman}. We decompose each trajectory~$\omega \in \Omega$ into two parts: the first transition resulting from taking a first action~$a$ and the rest of the trajectory~$\omega'$. The energy, the length and the likelihood of the trajectory can be decomposed in a similar way as the sum of the contribution of the first transition and the contribution of the rest of the trajectory. We get:
\begin{align*}
\mcZ(s,\beta) &= \sum_{\omega \in \Omega(s)} e^{-\beta E(\omega)+ \mu |\omega| +L(\omega) }\\
 &= \sum_{a, s'} e^{\beta \mcR(s,a,s')+ \mu +\log(\mathbb{P}(s'\mid{s,a})}\sum_{\omega' \in \Omega(s')}  e^{-\beta E(\omega')+ \mu |\omega'| +L(\omega') }\\
 &= \sum_{a, s'} e^{\beta \mcR(s,a,s')+ \mu +\log(\mathbb{P}(s'\mid{s,a})} ~ \mcZ(s',\beta) \\
 &= \sum_{a, s'} \mathbb{P}(s'\mid{s,a})~e^{\beta \mcR(s,a,s')+ \mu } ~ \mcZ(s',\beta) \\
 &=\sum_{a} \mathbb{E}_{s' \mid s,a} \left[e^{\beta \mcR(s,a,s') +\mu }~\mcZ(s',\beta)\right]
\end{align*}
This proves the equivalence.
\end{proof}

\subsection{Deriving the Unrealistic Bellman Equation}\label{proof:unrelasticBellman}
\begin{prop}
The value function~$V(s,\beta)= \frac{\partial}{\partial \beta} \log\mcZ(s,\beta)~$ where
\begin{equation*}
{\mcZ(s,\beta) = \sum_{\omega \in \Omega(s)} e^{-\beta E(\omega)+ \mu |\omega| + L(\omega) }} 
\end{equation*}
satisfies the following Bellman equation:
\begin{equation*}
V(s,\beta)  = \sum_{a,s'} \frac{e^{\beta \mcR(s,a,s') +\mu }~\mcZ(s',\beta)}{\mcZ(s,\beta)}~ \mathbb{P}(s' \mid s,a)  \left[  \mcR(s,a,s')+V(s',\beta) \right] 
\end{equation*}
\end{prop}

\begin{proof}
From Appendix~\ref{proof:avgBellman} we known that~$\mcZ(s,\beta)$ satisfies the Bellman equation: ${\mcZ(s,\beta) = \sum_{a} \mathbb{E}_{s' \mid s,a} \left[e^{\beta \mcR(s,a,s') +\mu }~\mcZ(s',\beta)\right]}$.
\begin{align*}
V(s,\beta) &= \frac{\partial}{\partial \beta} \log\mcZ(s,\beta) \\
&= \frac{\partial}{\partial \beta} \log \left(\sum_{a} \mathbb{E}_{s' \mid s,a} \left[e^{\beta \mcR(s,a,s') +\mu }~\mcZ(s',\beta)\right] \right)\\
&= \frac{1}{\mcZ(s,\beta)}~\frac{\partial}{\partial \beta} \left(\sum_{a} \mathbb{E}_{s' \mid s,a} \left[e^{\beta \mcR(s,a,s') +\mu }~\mcZ(s',\beta)\right] \right)\\
&= \frac{1}{\mcZ(s,\beta)}~ \sum_{a} \mathbb{E}_{s' \mid s,a} \left[ \frac{\partial}{\partial \beta} \left(e^{\beta \mcR(s,a,s') +\mu }~\mcZ(s',\beta)\right)\right] \\
&= \frac{1}{\mcZ(s,\beta)}~ \sum_{a} \mathbb{E}_{s' \mid s,a} \left[ \mcR(s,a,s')~ e^{\beta \mcR(s,a,s') +\mu }~\mcZ(s',\beta) +   e^{\beta \mcR(s,a,s') +\mu }~ \frac{\partial}{\partial \beta} \mcZ(s',\beta)\right] \\
&= \frac{1}{\mcZ(s,\beta)}~ \sum_{a} \mathbb{E}_{s' \mid s,a} \left[ \left( \mcR(s,a,s')+\frac{\frac{\partial}{\partial \beta} \mcZ(s',\beta)}{\mcZ(s',\beta)} \right)~ e^{\beta \mcR(s,a,s') +\mu }~\mcZ(s',\beta) \right] \\
&= \frac{1}{\mcZ(s,\beta)}~ \sum_{a} \mathbb{E}_{s' \mid s,a} \left[ \left( \mcR(s,a,s')+\frac{\partial}{\partial \beta} \log \mcZ(s',\beta)\right)~ e^{\beta \mcR(s,a,s') +\mu }~\mcZ(s',\beta) \right] \\
&= \frac{1}{\mcZ(s,\beta)}~ \sum_{a} \mathbb{E}_{s' \mid s,a} \left[ \left( \mcR(s,a,s')+V(s',\beta)\right)~ e^{\beta \mcR(s,a,s') +\mu }~\mcZ(s',\beta) \right] \\
&= \sum_{a,s'} \frac{e^{\beta \mcR(s,a,s') +\mu }~\mcZ(s',\beta)}{\mcZ(s,\beta)}~ \mathbb{P}(s' \mid s,a)  \left[  \mcR(s,a,s')+V(s',\beta) \right] 
\end{align*}
\end{proof}

\subsection{The Bellman operator of~$\mcZ(\rho,\beta)$ is a contraction } \label{proof:rhoContraction}

\begin{prop}
Let~${\mcD = \{\alpha \in \mathbb{R}^{|\mcS|} \text{ such that } ~\forall i \in {1,\cdots,|\mcS|}, ~ 0\leq \alpha_i \leq1 \text{ and } \sum_{i=1}^{,|\mcS|} \alpha_i =1 \}}$ and~${\mcX(\beta) = \left\{X\in C^{0}\left(\mcD,\mathbb{R}\right) \text { s.t.}~ X(\rho_f) = \exp\left[\beta \sum_{i=1}^M \alpha_i \mcR({{f_i}})\right] \text{ for mixtures of final states } \rho_f = \sum_{i=1}^M \alpha_i \delta_{{f_i}} \right\} }$. The map defined by  \begin{equation*}
         \psi: \begin{cases}
              \mcX(\beta)  &\to \mcX(\beta)   \\
               X    &\to  \begin{cases}
              \mcD  &\to \mathbb{R}   \\  \rho &\to \sum_{a} e^{\beta \mcR(\rho,a) +\mu } ~X({P_{a}}^T \rho,\beta) \end{cases}\\
           \end{cases}
\end{equation*}
is a contraction for the sup-norm: ~$\norm{X}_{\infty} = \underset{\rho \in \mcD}{ \max} ~\abs{X(\rho)}$.
\end{prop}
\begin{proof}

$\mcD~$ is the standard ~$(|\mcS|-1)$-simplex in~$\mathbb{R}^{|\mcS|}$ and~$\mcX(\beta)$ be the set of continuous functions on~$\mcD~$ satisfying the right boundary conditions. The original MDP is finite, consequently it has a finite number~$M$ of final state and it is possible to find a constant~$K$ such that, for all final states~$s_f$ we have~$\mcR(s_f) \leq K$. \\\\
Let ~$X_1$,~$X_2 \in   \mcX(\beta)$.~$X_1$ and~$X_2$ have the same boundary conditions by construction. Not only that, ~$\psi(X_1)$ and~$\psi(X_2)$ have also the same boundary conditions since the map~$\psi$ doesn't alter boundary conditions. Consequently we have:
\begin{align*}
\norm{\psi(X_1)-\psi(X_2)}_{\infty} = \underset{\rho \in \mcD}{\text{max}}~ \abs{\psi(X_1)(\rho)-\psi(X_2)(\rho)} = \underset{\rho \in \mcD, ~\rho \text{ non final}}{\text{max}}~~ \abs{\psi(X_1)(\rho)-\psi(X_2)(\rho)}
\end{align*}
Finally we can write:
\begin{align*}
\norm{\psi(X_1)-\psi(X_2)}_{\infty} &= \underset{\rho \in \mcD, ~\rho \text{ non final}}{\text{max}}~~ \abs{\psi(X_1)(\rho)-\psi(X_2)(\rho)}  \\
&= \underset{\rho \in \mcD, ~\rho \text{ non final}}{\text{max}}~~ \abs{\sum_{a} e^{\beta \mcR(\rho,a) +\mu } ~\left[X_1({P_{a}}^T \rho,\beta)-X_2({P_{a}}^T \rho,\beta)\right]} \\
&\leq \underset{\rho \in \mcD, ~\rho \text{ non final}}{\text{max}}~~ \abs{\sum_{a} e^{\beta \mcR(\rho,a) +\mu }} \times \norm{X_1-X_2}_{\infty} \\
&\leq \underbrace{d \times e^{\mu}}_{<1 ~\text{because} ~\mu < - \log d}  \norm{X_1-X_2}_{\infty}
\end{align*}
Where we use the fact that all rewards are non positive and that the number of available actions is bounded by~$d$. This concludes the proof that the Bellman operator of~$\mcZ(\rho,\beta)$ is a contraction.\\\\
This proof is generalization of the proof presented in Appendix~\ref{proof:contraction} for MDPs with finite state spaces.

\end{proof}

\end{document}